\documentclass{article}

\usepackage[preprint]{neurips_2024}

\usepackage[utf8]{inputenc} 
\usepackage[T1]{fontenc}    
\usepackage{amsmath}        
\usepackage{amsfonts}       
\usepackage{hyperref}       
\usepackage{url}            
\usepackage{booktabs}       
\usepackage{nicefrac}       
\usepackage{microtype}      
\usepackage{xcolor}         
\usepackage{amssymb}        
\usepackage{algorithm2e}
\usepackage{amsthm}         
\usepackage{graphicx}       

\newtheorem{theorem}{Theorem}

\newtheorem{lemma}{Lemma}
\newtheorem{remark}{Remark}

\newcommand\vn{\vec{\mathbf{n}}}

\newcommand\rc{{}^\mathcal{C}}
\newcommand\rw{{}^\mathcal{W}}

\newcommand\rcn{\rc\vec{\mathbf{n}}}
\newcommand\rwv{\rw\vec{\mathbf{v}}}

\title{Inteval Analysis for two spherical functions arising from robust Perspective-n-Lines problem}

\author{%
  Xiang Zheng \\
  School of Data Science\\
  Chinese University of Hong Kong(SZ)\\
  \texttt{224045013@link.cuhk.edu.cn} \\
  \And
  Haodong Jiang \\
  School of Data Science\\
  The Chinese University of Hong Kong(SZ)\\
  \texttt{haodongjiang@link.cuhk.edu.cn} \\
  \And
  Junfeng WU \\
  School of Data Science\\
  The Chinese University of Hong Kong(SZ)\\
  \texttt{junfengwu@cuhk.edu.cn}
}

\begin{document}

\maketitle

\begin{abstract}
  This report presents a comprehensive interval analysis of two spherical functions derived from the robust Perspective-n-Lines (PnL) problem. The study is motivated by the application of a dimension-reduction technique to achieve global solutions for the robust PnL problem. We establish rigorous theoretical results, supported by detailed proofs, and validate our findings through extensive numerical simulations.
\end{abstract}

\section{Preliminary}
\textbf{Notations:} we use the notation $ \mathbf{a} \bullet\mathbf{b}\triangleq \mathbf{a}^\top\mathbf{b} $, and use the notation 
$ (\mathbf{a},\mathbf{b})\triangleq\begin{bmatrix}
  \mathbf{a} \\ \mathbf{b}
\end{bmatrix} $ for the concatenation of two vectors. We use the notation ${}^{\mathcal{F}}\mathbf{a}$ to highlight that $\mathbf{a}$ is observed in the reference frame ${\mathcal{F}}$. Specifically, we denote the normalized camera frame as ${\mathcal{C}}$, and denote the world frame as ${\mathcal{W}}$.  
\subsection{Parameterization of Lines}
Consider a 2D line in the image which writes as follows in the pixel coordinate: 
$$
\begin{bmatrix}
  A&B&C
\end{bmatrix}(u,v,1) = 0.
$$
where the coefficients can be easily determined with two pixels on the line. According to the following linear transformation:
$$
(x,y,1)
= \mathbf{K}^{-1}(u,v,1),
$$
where $\mathbf{K}$ is the camera intrinsic matrix, we can write the same line in the normalized image coordinate as
$$
\begin{bmatrix}
  A_c&B_c&C_c
\end{bmatrix}(x,y,1)=0,
$$
with $
\begin{bmatrix}
  A_c&B_c&C_c
\end{bmatrix}=\begin{bmatrix}
  A&B&C
\end{bmatrix}\mathbf{K} 
$. We use the normalized coefficient vector $\vn$ to parameterize a 2D line $l_c$ in the normalized camera coordinate:
$$
\rc\vn=\frac{(A_c,B_c,C_c)}{\|(A_c,B_c,C_c)\|},~~l_c:=\{\rc\mathbf{p}\in\mathbb{R}^2|\rc\vn\bullet(\rc\mathbf{p},1)=0\}.
$$
We refer to $\rc\vn$ as the normal vector since it is perpendicular to the plane passing through the camera origin and $l_c$. 

As for a 3D line $L_w$ observed in the world coordinate, we parameterize it with a point $\rw\mathbf{p}_0$ on it and a unit-length direction vector $\rw\vec{\mathbf{v}}$, such that
$$
L_w:=\{\rw\mathbf{p}\in\mathbb{R}^3|(\rw\mathbf{p}-\rw\mathbf{p}_0)\times\rw\vec{\mathbf{v}}=0\}.
$$
\subsection{Projection Model}
Assume the relative transformation from the normalized camera frame to the world coordinate writes as follows
$$
{}^{\mathcal{W}}_{\mathcal{C}}\mathbf{T}=\begin{bmatrix}
    \mathbf{R} & \mathbf{t}\\\mathbf{0}^\top & 1
\end{bmatrix}.
$$ Assume a 2D line $l_c$ with normal vector $\rcn$ is the projection of a 3D line $L_w$ parameterized with $\rwv$ and $\rw\mathbf{p}_0$, the following two equations~\cite{liu1990determination} uniquely determine the projection:
\begin{align}
    \rc\vec{\mathbf{n}} \bullet \mathbf{R}^\top\rw\vec{\mathbf{v}} &= 0, \label{eqn::rotation_equation}\\
    (\mathbf{R}\rc\vec{\mathbf{n}})\bullet(\rw\mathbf{p}_0-\mathbf{t})&= 0.\label{eqn::translation_equation}
\end{align}
\subsection{Accelerating Consensus Maximization}
Consider the following 1d CM problem:
\begin{equation}\label{model::1D_CM}
    \max_{b\in\mathcal{X}}\sum_{k=1}^{K}\mathbf{1}\{|f(b|\mathbf{s}_k)|\leq\epsilon\},
\end{equation}
where $b$ is a scalar parameter that belongs to $\mathcal{X}\subseteq\mathbb{R}$, $\mathbf{s}_k$ is data, and $f(b|\mathbf{s}_k)$ is the residual function continuous in $b$. Suppose we can obtain $
|f(b|\mathbf{s}_k)|\leq\epsilon \Leftrightarrow b\in \bigcup_{l}[b_{kl}^l,b_{kl}^r]
$, the authors of ~\cite{zhang2024accelerating} observe that problem~\eqref{model::1D_CM} is equivalent to the following interval stabbing problem:
\begin{equation}\label{model::1d_stab}
\max_{b\in\mathcal{X}}\sum_{k=1}^{K}\sum_l\mathbf{1}\{b\in[b_{kl}^l,b_{kl}^r]\}.
\end{equation}
Next, consider a n-d CM problem: 
\begin{equation}\label{model::nD_CM}
    \max_{\mathbf{b}\in\mathcal{X}\subseteq\mathbb{R}^n}\sum_{k=1}^{K}\mathbf{1}\{|f(\mathbf{b}|\mathbf{s}_k)|\leq\epsilon\}.
\end{equation}
The ACM method distinguishes one parameter with others, $\mathbf{b}=(\mathbf{b}_{1:n-1},b_n)$, and branches only the space of $\mathbf{b}_{1:n-1}$. This is achieved by revising the bound-seeking procedure. Suppose we are seeking bounds for $\mathbf{b}_{1:n-1}\in\boldsymbol{\mathcal{C}}_{1:n-1}$ and $b_n$.
\subsubsection{Lower Bound}
Denote the consensus maximizer as $(\mathbf{b}^*_{1:n-1},b_n^*)$, ACM finds a lower bound as follows, $\sum_{k=1}^{K}\mathbf{1}\{|f(\mathbf{b}_{1:n-1}^*,b_n^*|\mathbf{s}_k)|\leq\epsilon\}$
$$
\geq~\max_{b_n}\sum_{k=1}^{K}\mathbf{1}\{|f(b_n|\mathbf{b}_{1:n-1}^{(c)},\mathbf{s}_k)|\leq\epsilon\},
$$
where $\mathbf{b}_{1:n-1}^{(c)}$ is the center point of $\boldsymbol{\mathcal{C}}_{1:n-1}$. Notice that the lower bound corresponds to a 1-d CM problem, and it can be efficiently solved by interval stabbing as we do in~\eqref{model::1d_stab}. 
\subsubsection{Upper Bound}
If a consensus problem can be written as 
\begin{equation}
    \max_{\mathbf{b}\in\mathcal{X}}\sum_{k=1}^{K}\mathbf{1}\{|\sum_{i} f_i(b_n,h_i(\mathbf{b}_{1:n-1},\mathbf{s}_k)|\mathbf{s}_k)|\leq\epsilon\}\label{model::sum_CM}
\end{equation}
where $h_i(\mathbf{b}_{1:n-1})$ is a function of $\mathbf{b}_{1:n-1}$. And if all $f_i$ are monotonically increasing(\textit{it's similar to decreaseing condition}) in $h_i$, set that 
\begin{equation}
    h_i^{L} \le h_i(\mathbf{b}_{1:n-1},\mathbf{s}_k) \le h_i^{U}  
\end{equation}
we can get 
\begin{equation}
  f_L(b_n|\mathbf{s}_k)=\sum_{i} f_i(b_n|h_i^{L},\mathbf{s}_k) \le \sum_{i} f_i(b_n,h_i(\mathbf{b}_{1:n-1},\mathbf{s}_k)|\mathbf{s}_k)\le \sum_{i} f_i(b_n|h_i^{U},\mathbf{s}_k)=f_U(b_n|\mathbf{s}_k)
\end{equation}

Given these bounding functions, ACM finds an upper bound as follows, $\sum_{k=1}^{K}\mathbf{1}\{|f(\mathbf{b}_{1:n-1}^*,b_n^*|\mathbf{s}_k)|\leq\epsilon\}$
\begin{equation}
  \label{eqn::upper_interval}
  \leq\max_{b_n}\sum_{k=1}^{K}\mathbf{1}\left[\{f_L(b_n|\mathbf{s}_k)\leq\epsilon\}\bigcap\{f_U(b_n|\mathbf{s}_k)\geq-\epsilon\}\right]
\end{equation}

which can be solved by interval stabbing.

Readers can refer to \cite{zhang2024accelerating}'s work for the detailed introduction and other applications of the accelerating consensus maximization algorithm.

\section{Problem Formulation}
\subsection{Basic problem}
Given a set of 3D lines $\{L_{w_i}\}_{i=1}^N$ and their corresponding 2D lines $\{l_{c_i}\}_{i=1}^N$, the CM problem for the PnL problem can be formulated as follows:
\begin{equation}\label{model::plain_rotation_CM}
  \mathop{\rm max}\limits_{\mathbf{R}\in SO(3)} ~~\sum_k \boldsymbol{1}\{|\mathbf{}\rc\vec{\mathbf{n}}_k\bullet\mathbf{R}^\top\rw\vec{\mathbf{v}}_{k}|<=\epsilon_r\},
\end{equation}
where $\vec{\mathbf{v}}_k$ represent the direction of 3D lines, $\vec{\mathbf{n}}_k$ represent the normalized coefficient vector of 2D lines.
\subsection{How to accelerate?}
In the rotation estimation problem~\eqref{model::plain_rotation_CM}, we parameterize rotation with a rotation axis $\vec{\mathbf{u}}\in\mathbb{S}^3$ and an amplitude $\theta\in[0,\pi]$. We choose $\theta$ as the distinguished parameter, and further parameterize $\vec{\mathbf{u}}$ by polar coordinates: 
$$
\vec{\mathbf{u}}=(\sin{\alpha}\cos{\phi}, \sin{\alpha}\sin{\phi},\cos{\alpha})~~ \alpha\in[0,\pi]~\phi\in[0,2\pi].
$$
Denote data from a pair of 2D/3D line matching as $\mathbf{s}_k\triangleq(\vec{\mathbf{v}}_k,\vec{\mathbf{n}}_k)$, we can write the observation function for~\eqref{model::plain_rotation_CM} as
\begin{equation}\label{eqn::residual}
  \mathbf{}\rc\vec{\mathbf{n}}_k\bullet\mathbf{R}^\top\rw\vec{\mathbf{v}}_{k}= \vec{\mathbf{n}}_k^\top  \vec{\mathbf{v}}_k +\sin{\theta}\vec{\mathbf{n}}_k^\top(\vec{\mathbf{u}}\times \vec{\mathbf{v}}_k)+(1-\cos{\theta})\vec{\mathbf{n}}_k^\top[\vec{\mathbf{u}}]_{\times}^2\vec{\mathbf{v}}_k,
\end{equation}
Rearrange the terms in~\eqref{eqn::residual} likes \eqref{model::sum_CM}, we can get
\begin{equation}
  \label{eqn::h1_h2}
  \begin{cases}
    f_1(\theta ,h_1(\vec{\mathbf{u}},\mathbf{s}_k),s_k)=\vec{\mathbf{n}}_k^\top  \vec{\mathbf{v}}_k +\sin{\theta}h_1 &\text{where } h_1(\vec{\mathbf{u}}|\mathbf{s}_k)\triangleq
  \vec{\mathbf{u}}^\top( \vec{\mathbf{v}}_k\times\vec{\mathbf{n}}_k)   \\
  
  f_2(\theta ,h_2(\vec{\mathbf{u}},\mathbf{s}_k),s_k)= (1-\cos \theta )h_2 &\text{ where }h_2(\vec{\mathbf{u}}|\mathbf{s}_k)\triangleq\vec{\mathbf{n}}_k^\top[\vec{\mathbf{u}}]_{\times}^2\vec{\mathbf{v}}_k  
  \end{cases}
\end{equation}
As long as we can find the lower and upper bounds for $h_1$ and $h_2$, we can find the accessible intervals for upper bound of \eqref{model::nD_CM} according to \eqref{eqn::upper_interval}.
\subsection{conclusions}
For clarity, we denote the sub-cube as
$$
\boldsymbol{\mathcal{C}}_{\vec{\mathbf{u}}}\triangleq\{(\alpha,\phi)|\alpha\in[\underline{\alpha},\bar{\alpha}],\phi\in[\underline{\phi},\bar{\phi}]\},
$$
denote the boundary of $\boldsymbol{\mathcal{C}}_{\vec{\mathbf{u}}}$ as $\partial \boldsymbol{\mathcal{C}}_{\vec{\mathbf{u}}}$, and denote
$$
\vec{\mathbf{m}}_k = \frac{\vec{\mathbf{v}}_k+\vec{\mathbf{n}}_k}{\|\vec{\mathbf{v}}_k+\vec{\mathbf{n}}_k\|},~~\vec{\mathbf{m}}_k^\perp=\frac{\vec{\mathbf{v}}_k-\vec{\mathbf{n}}_k}{\|\vec{\mathbf{v}}_k-\vec{\mathbf{n}}_k\|},~~\vec{\mathbf{c}}_k\triangleq\frac{\vec{\mathbf{v}}_k\times\vec{\mathbf{n}}_k}{\|\vec{\mathbf{v}}_k\times\vec{\mathbf{n}}_k\|}.
$$
We summarize our results in the following theroems:
\subsubsection{Theorem and proof of $h_2$}

\begin{theorem}[Extreme Point Theorem for $h_2(\vec{\mathbf{u}}|\mathbf{s}_k)$]\label{Theorem::h2}
  ~\\
  \begin{enumerate}
          \item If $\pm\vec{\mathbf{m}}_k\in\boldsymbol{\mathcal{C}}_{\vec{\mathbf{u}}}$, we have $\arg\max_{\vec{\mathbf{u}}}h_2(\vec{\mathbf{u}}|\mathbf{s}_k)=\pm\vec{\mathbf{m}}_k.$ 
          \item If $\pm\vec{\mathbf{m}}_k^\perp\in\boldsymbol{\mathcal{C}}_{\vec{\mathbf{u}}}$, we have $\arg\min_{\vec{\mathbf{u}}}h_2(\vec{\mathbf{u}}|\mathbf{s}_k)=\pm\vec{\mathbf{m}}_k^\perp$.
          \begin{figure}[htbp]
            \centering
            \includegraphics[width=0.8\textwidth]{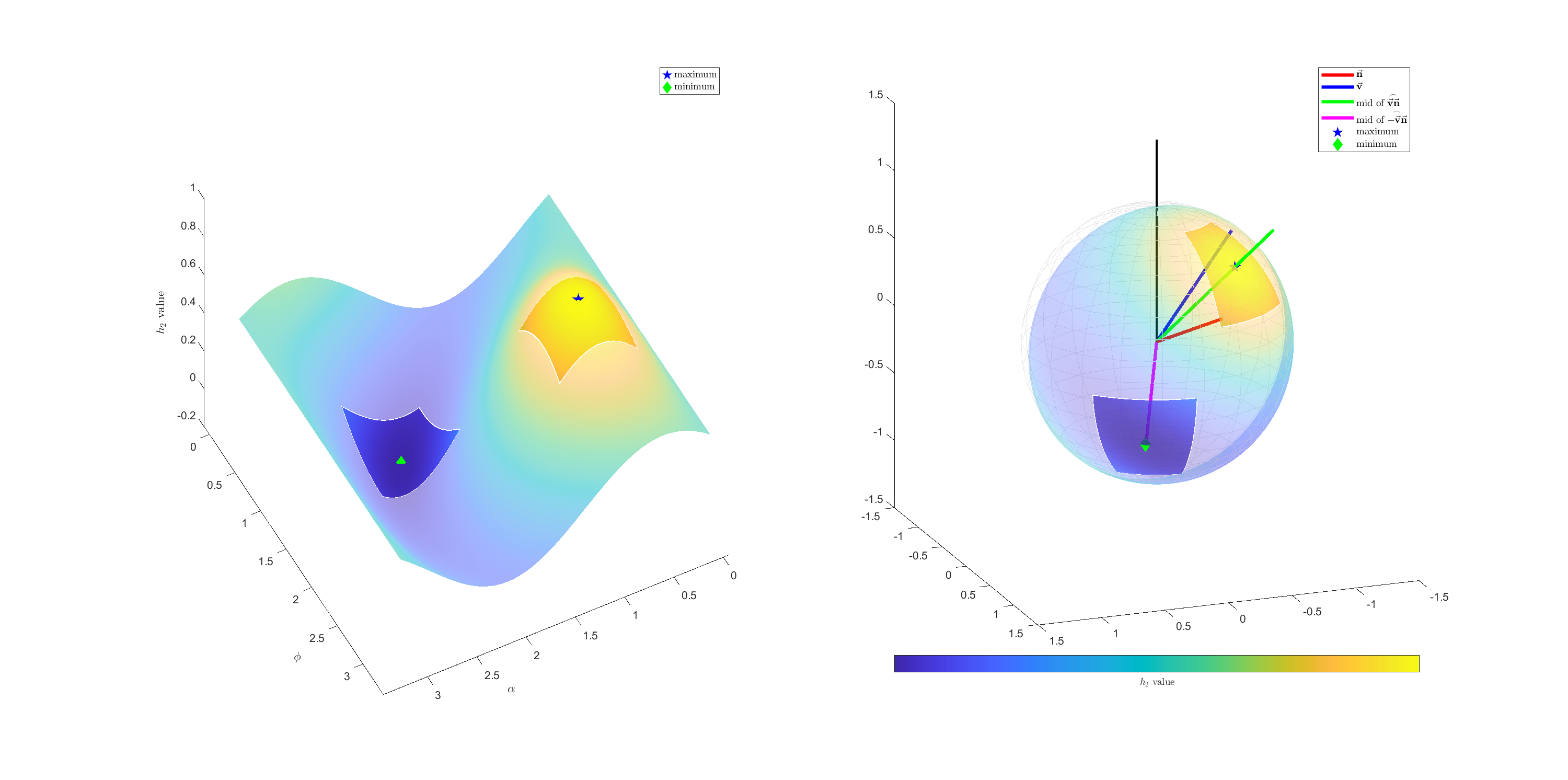}
            \caption{Illustration of the first and second cases in Theorem~\ref{Theorem::h2}.}
          \end{figure} 
          \item Otherwise, the extreme points fall on $\partial \boldsymbol{\mathcal{C}}_{\vec{\mathbf{u}}}$.
          \begin{figure}[htbp]
            \centering
            \includegraphics[width=0.8\textwidth]{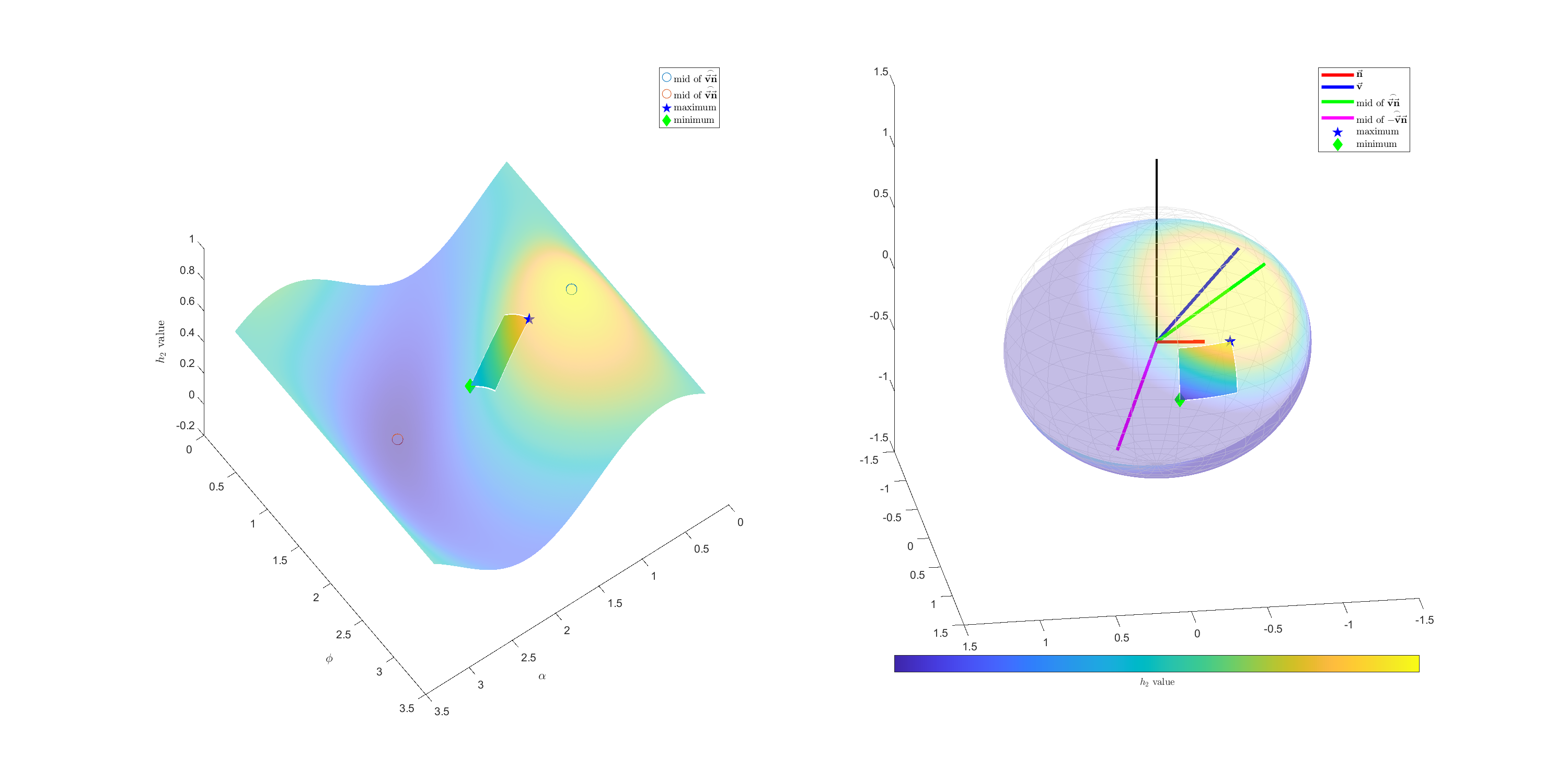}
            \caption{Illustration of the third case in Theorem~\ref{Theorem::h2}.}
          \end{figure}
      \end{enumerate}
  \end{theorem}
\begin{proof}
   TBD
\end{proof}

\subsubsection{Theorem and proof of $h_1$}
\begin{theorem}[Extreme Point Theorem for $h_1(\vec{\mathbf{u}}|\mathbf{s}_k)$]\label{Theorem::h1}
  ~\\
  \begin{enumerate}
          \item If $\vec{\mathbf{c}}_k\in\boldsymbol{\mathcal{C}}_{\vec{\mathbf{u}}}$, we have $\arg\max_{\vec{\mathbf{u}}}h_1(\vec{\mathbf{u}}|\mathbf{s}_k)=\vec{\mathbf{c}}_k.$ 
          \item If $-\vec{\mathbf{c}}_k\in\boldsymbol{\mathcal{C}}_{\vec{\mathbf{u}}}$, we have $\arg\min_{\vec{\mathbf{u}}}h_1(\vec{\mathbf{u}}|\mathbf{s}_k)=-\vec{\mathbf{c}}_k$.
          \begin{figure}[htbp]
            \centering
            \includegraphics[width=0.6\textwidth]{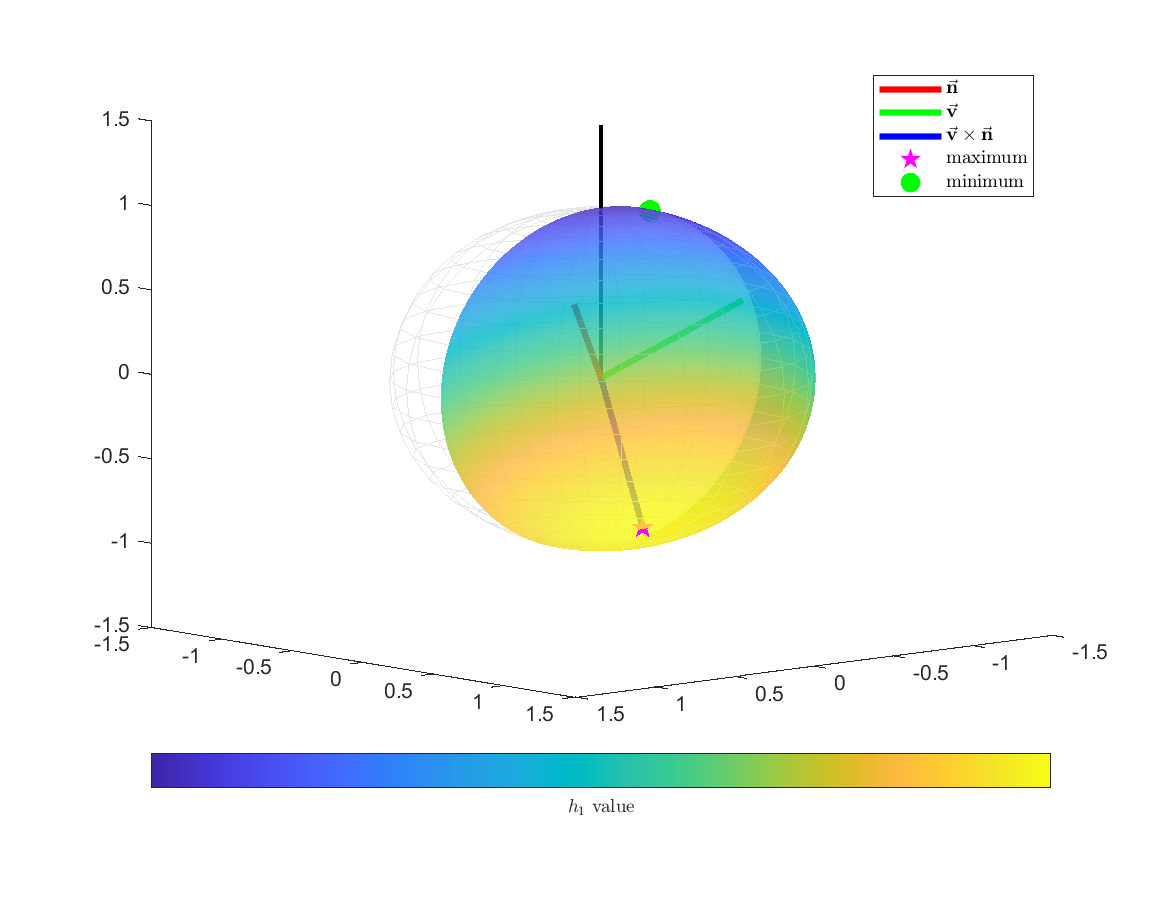}
            \caption{Illustration of the first and second cases in Theorem~\ref{Theorem::h1}.}
          \end{figure}
          \item Otherwise, the extreme points fall on $\partial \boldsymbol{\mathcal{C}}_{\vec{\mathbf{u}}}$.
          \begin{figure}[htbp]
            \centering
            \includegraphics[width=0.6\textwidth]{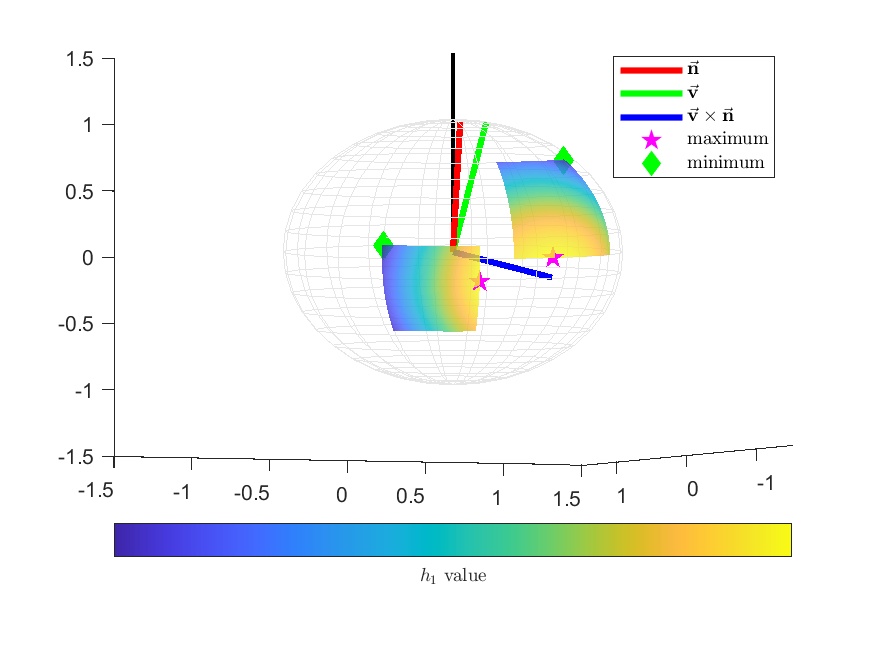}
            \caption{Illustration of the third case in Theorem~\ref{Theorem::h1}.}
          \end{figure}
      \end{enumerate}
  \end{theorem}

The proof for Theorem~\ref{Theorem::h1} is similar to and easier than the proof for Theorem~\ref{Theorem::h2}, thus we omit the proof here and focus on discussing the third case in Theorem~\ref{Theorem::h1}. Denote the polar coordinates for $\vec{\mathbf{c}}_k$ as $(\alpha_k,\phi_k)$, we have:
\begin{equation}\label{accelerated_BnB::rotation::h1_proof}
h_1(\vec{\mathbf{u}}|\mathbf{s}_k)\propto(\sin{\alpha_k}\sin{\alpha}\cos{(\phi_k- \phi)} + \cos{\alpha_k}\cos{\alpha}), 
\end{equation}
where the constant term $\|\vec{\mathbf{v}}_k\times \vec{\mathbf{n}}_k\|$ is omitted. The partial derivative of $h_1$ with respect to $\alpha$ and $\phi$ writes
  \begin{subequations}
  \begin{align}
    \frac{\partial h_1}{\partial \alpha}&\propto\sin{\alpha_k}\cos{\alpha}\cos{(\phi_k - \phi)}-\sin{\alpha}\cos{\alpha_k},\label{eqn::h1_partial_alpha}\\
    \frac{\partial h_1}{\partial \phi}&\propto\sin{\alpha_k}\sin{\alpha}\sin{(\phi_k - \phi)}.
    \label{eqn::h1_partial_phi}
  \end{align}
  \end{subequations}
Without loss of generality, we only consider the case where both the sub-cube $\boldsymbol{\mathcal{C}}_{\vec{\mathbf{u}}}$ and $\vec{\mathbf{c}}_k$ belong to the east-hemisphere, i.e., $\phi_k\in[0,\pi]$ and $\phi\in[0,\pi]$. Denote polar coordinates which minimize $h_1$ as $\alpha_{\rm min}$ and $\phi_{\rm min}$ respectively, and denote the maximizers as $\alpha_{\rm max}$ and $\phi_{\rm max}$ respectively. Use notations $\alpha_{\rm near}$ and $\phi_{\rm near}$ as
$$
\alpha_{\rm near}\triangleq\mathop{\arg\min}_{\alpha\in[\alpha_l,\alpha_r]} |\alpha-\alpha_k|,~~\phi_{\rm near}\triangleq\mathop{\arg\min}_{\phi\in[\phi_l,\phi_r]} |\phi-\phi_k|.
$$
Note that $\alpha_{\rm near}=\alpha_k$ if $\alpha_k\in[\alpha_l,\alpha_r]$. Similarly, we use notations $\alpha_{\rm far}$ and $\phi_{\rm far}$. We first give a lemma for $\phi$:
\begin{lemma}\label{lemma::phi}
    If $(\alpha,\phi)\in\partial\boldsymbol{\mathcal{C}}_{\vec{\mathbf{u}}}$ is a extreme point for $h_1(\vec{\mathbf{u}}|\mathbf{s}_k)$ on the boundaries of cube, one must have
    \label{lemma::h1_phi}
    $$
    \phi_{\rm max}=\phi_{\rm near}, \phi_{\rm min}=\phi_{\rm far}.
    $$
    \begin{proof}
        Based on the partial derivative~\eqref{eqn::h1_partial_phi}, we highlight two observations.~(1) For a fixed $\alpha\in[0,\pi]$, we have $\frac{\partial h_1}{\partial \phi}>0$ if $\phi<\phi_k$, and $\frac{\partial h_1}{\partial \phi}<0$ if $\phi>\phi_k$.~(2) The partial derivative takes the same value for $\phi_1$ and $\phi_2$ equally close to $\phi_k$. Based on the above two observations, we naturally conclude this lemma.
    \end{proof}
\end{lemma}
After we fix $\phi$ at either $\phi_{\rm near}$ or $\phi_{\rm far}$, we focus on the partial derivative ~\eqref{eqn::h1_partial_alpha}. 
\begin{lemma}\label{lemma::alpha}
For $\alpha_k\not=\pi/2$, the partial derivative ~\eqref{eqn::h1_partial_alpha}.  has a unique zero point $\alpha^*\in[0,\pi]$. For a fixed $\phi$, the zero point $\alpha^*$ is a global maximizer if $|\phi_k-\phi|<\pi/2$, and is a global minimizer if $|\phi_k-\phi|<\pi/2$.
\begin{proof}
The partial derivative~\eqref{eqn::h1_partial_alpha} can be organized in the form of $A\sin{(\alpha+\bar{\beta})}$, with $\bar{\beta}$ as fixed angle. For $\alpha_k\not=\pi/2$, we have $\bar{\beta}\not=0$, and as a result there exist a unique zero point $\alpha^*\in[0,\pi]$. If $|\phi_k-\phi|\not=\pi/2$ and $\alpha\not=\pi/2$, we can rewrite \eqref{eqn::h1_partial_alpha} as:
\begin{equation}\label{eqn::h1_partial_alpha::transform}
    \frac{\partial h_1}{\partial \alpha}= cos{\alpha_k}cos{\alpha}(\tan{\alpha_k}\cos{(\phi_k-\phi)}-\tan{\alpha}).
\end{equation}
We discuss four cases in the table below, and the results are easy to verify using~\eqref{eqn::h1_partial_alpha::transform}. From the table, we can observe that $\cos\alpha_k\cos\alpha^*>0$ for $\Delta\phi<\pi/2$, and $\cos\alpha_k\cos\alpha^*<0$ for $\Delta\phi>\pi/2$. We naturally arrive at the conclusion in this lemma based on this observation.
\begin{table}[htbp]
\centering
\caption{ Four different cases of $\alpha_k$ and $\Delta \phi$}
    \begin{tabular}{|l|l|l|c|}
        \hline
        $\alpha_k$& $\Delta \phi$ & $\alpha^*\in$  & $\cos \alpha_k\cos \alpha^*$ \\
        \hline
        $<\pi/2$&  $<\pi/2$ &  $(0,\alpha_k)$&  $>0$\\ \hline
        $>\pi/2$&  $<\pi/2$ &  $(\alpha_k,\pi)$&  $>0$ \\ \hline
        $<\pi/2$&  $>\pi/2$ &  $(\pi-\alpha_k,\pi)$&  $<0$ \\ \hline
        $>\pi/2$&  $>\pi/2$ &  $(0,\pi-\alpha_k)$&  $<0$\\ \hline
    \end{tabular}
    \label{h1::table}
\end{table}
\end{proof}
\end{lemma}
\begin{remark}
    Notice that we omit to discuss the special cases where $\alpha_k=\pi/2$ or $|\phi-\phi_k|=\pi/2$ for the sake of simplicity. These special cases are easy to handle, interested readers can refer to our code for details.
\end{remark}
Combining Lemma~\ref{lemma::alpha} and Lemma~\ref{lemma::phi}, we propose a efficient procedure to find extreme points of $h_1(\vec{\mathbf{u}}|\mathbf{s}_k)$ on $\partial\boldsymbol{\mathcal{C}}_{\vec{\mathbf{u}}}$. First of all, we find the maximizer with $\phi=\phi_{\rm near}$. Denote $\Delta\phi_{\rm near}\triangleq |\phi_k-\phi_{\rm near}|$.
\begin{enumerate}
    \item If $\Delta\phi_{\rm near}=0$, we have
    $$
    \alpha_{\rm max}=\alpha_{\rm near}.
    $$
    \item if $\Delta\phi_{\rm near}=\pi/2$
        $$\alpha_{\max}=\begin{cases}
\alpha _1 &\text{if} \quad \alpha \le \pi/2 \\
\alpha _2 &\text{if} \quad \alpha > \pi/2

\end{cases}$$
    \item If $\Delta\phi_{\rm near}>\pi/2$, we have    
    $$
    \alpha_{\rm max}= \mathop{\arg\max}_{\alpha\in[\alpha_l,\alpha_r]}|\alpha-\alpha^*(\Delta\phi_{\rm near})|.
    $$
    \item If $\Delta\phi_{\rm near}<\pi/2$, $\alpha_k<\pi/2$, and $\alpha_l>=\alpha_k$, we have
    $$
    \alpha_{\rm max}=\alpha_l.
    $$
    \item $\Delta\phi_{\rm near}<\pi/2$, $\alpha_k>\pi/2$, and $\alpha_r<=\pi-\alpha_k$, we have
    $$
    \alpha_{\rm max}=\alpha_r.
    $$
    \item Otherwise, calculate $\alpha^*(\Delta\phi_{\rm near})$ and we have
    $$
    \alpha_{\rm max}=\mathop{\arg\min}_{\alpha\in[\alpha_l,\alpha_r]}|\alpha-\alpha^*(\Delta\phi_{\rm near})|.
    $$ 
\end{enumerate}
We can find the minimizer with $\phi=\phi_{\rm far}$ with a quite symmetrical procedure. Denote $\Delta\phi_{\rm far}=|\phi_k-\phi_{\rm far}|$.
\begin{enumerate}
    \item If $\Delta\phi_{\rm far}<\pi/2$, we have    
    $$
    \alpha_{\rm min}= \mathop{\arg\min}_{\alpha\in[\alpha_l,\alpha_r]}|\alpha-\alpha^*(\Delta\phi_{\rm far})|.
    $$
    \item if $\Delta\phi_{\rm near}=\pi/2$
    $$
        \alpha_{\min}=\begin{cases}
        \alpha _2 &\text{if} \quad \alpha \le \pi/2 \\
        \alpha _1 &\text{if} \quad \alpha > \pi/2
        
        \end{cases}
    $$
    \item If $\Delta\phi_{\rm far}>\pi/2$, $\alpha_k<\pi/2$, and $\alpha_r<=\pi-    \alpha_k$, we have
    $$
    \alpha_{\rm min}=\alpha_r.
    $$
    \item $\Delta\phi_{\rm far}>\pi/2$, $\alpha_k>\pi/2$, and $\alpha_l>=\pi-\alpha_k$, we have
    $$
    \alpha_{\rm min}=\alpha_l.
    $$
    \item Otherwise, calculate $\alpha^*(\Delta\phi_{\rm far})$ and we have
    $$
    \alpha_{\rm min}=\mathop{\arg\min}_{\alpha\in[\alpha_l,\alpha_r]}|\alpha-\alpha^*(\Delta\phi_{\rm far})|.
    $$ 
\end{enumerate}

{
\small
\bibliographystyle{abbrvnat}
\bibliography{reference.bib}
}






\end{document}